\documentclass[conference]{IEEEtran}
\IEEEoverridecommandlockouts
\usepackage{cite}
\usepackage{amsmath,amssymb,amsfonts}
\usepackage{comment}
\usepackage{graphicx}
\usepackage{textcomp}
\usepackage{xcolor}
\def\BibTeX{{\rm B\kern-.05em{\sc i\kern-.025em b}\kern-.08em
    T\kern-.1667em\lower.7ex\hbox{E}\kern-.125emX}}
    
\usepackage[ruled,vlined,linesnumbered]{algorithm2e}

\SetCommentSty{mycommfont}

\usepackage{amsmath,amssymb}
\allowdisplaybreaks

\usepackage{xfrac}

\usepackage{amsthm}

\newtheorem{assumption}{Assumption}
\newtheorem{theorem}{Theorem}
\newtheorem{corollary}{Corollary}

\usepackage{hyperref}

\newcommand{\Identity}{{\rm I\kern-.2em l}}

\newcommand{\Expectbracket}[1]{\mathbb{E}\left[ #1 \right]}
\newcommand{\Expectbracketsizable}[2]{\mathbb{E}#1[ #2 #1]}
\newcommand{\Expectsubbracket}[2]{\mathbb{E}_{#1}\left[ #2 \right]}
\newcommand{\Expectsubbracketsizable}[3]{\mathbb{E}_{#1}#2[ #3 #2]}
\newcommand{\Expectcond}[2]{\mathbb{E}\left[\left. #1 \right| #2 \right]}
\newcommand{\bx}{\mathbf{x}}
\newcommand{\by}{\mathbf{y}}
\newcommand{\bz}{\mathbf{z}}
\newcommand{\bg}{\mathbf{g}}
\newcommand{\ba}{\mathbf{a}}
\newcommand{\bb}{\mathbf{b}}
\newcommand{\be}{\mathbf{e}}
\newcommand{\br}{\mathbf{r}}
\newcommand{\bu}{\mathbf{u}}
\newcommand{\bv}{\mathbf{v}}
\newcommand{\norm}[1]{\left\Vert #1 \right\Vert}
\newcommand{\normsq}[1]{\left\Vert #1 \right\Vert^2}
\newcommand{\normsqsizable}[2]{#1\Vert #2 #1\Vert^2}

\newcommand{\innerprod}[1]{\left\langle #1 \right\rangle}
    
\newcommand{\initqueue}{W}

\begin{document}

\title{Federated Learning with Flexible Control \vspace{-0.05in}}

\author{
\IEEEauthorblockN{Shiqiang Wang\IEEEauthorrefmark{1}, Jake Perazzone\IEEEauthorrefmark{2}, Mingyue Ji\IEEEauthorrefmark{3}, Kevin S. Chan\IEEEauthorrefmark{2}}
\IEEEauthorblockA{
\IEEEauthorrefmark{1}IBM T. J. Watson Research Center, Yorktown Heights, NY, USA. Email: wangshiq@us.ibm.com\\
\IEEEauthorrefmark{2}Army Research Laboratory, Adelphi, MD, USA. Email: \{jake.b.perazzone.civ; kevin.s.chan.civ\}@army.mil, \\
\IEEEauthorrefmark{3}Department of ECE, University of Utah, Salt Lake City, UT, USA. Email: mingyue.ji@utah.edu
}
\thanks{This research was partly sponsored by the U.S. Army Research Laboratory under Agreement Number W911NF-16-3-0002 and the National Science Foundation (NSF) CAREER Award 2145835. The views and conclusions contained in this document are those of the authors and should not be interpreted as representing the official policies, either expressed or implied, of the U.S. Army Research Laboratory or the U.S. Government. The U.S. Government is authorized to reproduce and distribute reprints for Government purposes notwithstanding any copyright notation hereon.}
\vspace{-0.25in}
}

\maketitle

\begin{abstract}
Federated learning (FL) enables distributed model training from local data collected by users. In distributed systems with constrained resources and potentially high dynamics, e.g., mobile edge networks, the efficiency of FL is an important problem. Existing works have separately considered different configurations to make FL more efficient, such as infrequent transmission of model updates, client subsampling, and compression of update vectors. However, an important open problem is how to jointly apply and tune these control knobs in a single FL algorithm, to achieve the best performance by allowing a high degree of freedom in control decisions. In this paper, we address this problem and propose \textit{FlexFL} -- an FL algorithm with multiple options that can be adjusted flexibly. Our FlexFL algorithm allows both arbitrary rates of local computation at clients and arbitrary amounts of communication between clients and the server, making both the computation and communication resource consumption adjustable. We prove a convergence upper bound of this algorithm. Based on this result, we further propose a stochastic optimization formulation and algorithm to determine the control decisions that (approximately) minimize the convergence bound, while conforming to constraints related to resource consumption. The advantage of our approach is also verified using experiments.
\end{abstract}


\section{Introduction}
Many emerging applications nowadays are driven by machine learning technologies. To train models that are used in such applications, a large training dataset is usually needed. However, it has become increasingly common that data are collected and stored by local users at their end devices or organizational servers. It is difficult to share such data with a central entity, due to privacy regulations and communication bandwidth limitation. As a result, \textit{federated learning (FL)} has emerged as a promising technique for distributed model training from decentralized local datasets~\cite{kairouz2021advances,li2020federated,yang2019federated}.

At its core, FL includes model updates at each client (e.g., user device) using its own local data and aggregation of model parameters through a server (e.g., a cloud instance). In a resource-constrained system, such as a mobile edge network, these FL operations consume both computation and communication resources. Therefore, an important research direction is \textit{how to make the most efficient use of the limited resources to maximize the performance of FL}. Some recent works have considered this problem by tuning configuration parameters of the FL algorithm, such as the number of local updates in each FL round, participation rate of clients, and compression rate of parameters transmitted between clients and the server \cite{hsieh2017gaia,WangJSAC2019,MLSYS2019Jianyu,han2020adaptive,li2020ggs,li2021talk,abdelmoniem2021dc2,xu2021grace,cui2022optimal,nishio2018client,CMFL,wang2020optimizing,shi2020joint,perazzone2022communication,luo2021cost,luo2022tackling,wu2022node}. However, the vast majority of them only focus on adjusting a small subset of all the available control options in FL, which cannot achieve the full potential of making FL the most efficient. In particular, the automatic adaptation of both parameter compression (e.g., sparsification and quantization) and partial client participation has not been studied, to the best of our knowledge. There is also usually a tight coupling between computation and communication in existing works, which may be difficult to achieve in heterogeneous systems where the costs of different resources can vary over time.

In light of these limitations, there is an important open problem: \textit{Is it possible to jointly apply a wide range of control options in a single FL algorithm, to support heterogeneous and time-varying costs\footnote{We consider the resource \textit{cost} as a generic metric in this paper, which can be defined as related to the availability of each type of resource.} of multiple types of resources?} 
There are several challenges in answering this question. \textit{1)} It is non-straightforward to design an FL algorithm that allows simultaneous adjustment of multiple configurations with a high degree of freedom. \textit{2)} It is difficult to analyze and understand the influence of different control options and the interplay between them on the FL performance. \textit{3)} It is challenging to design an efficient control algorithm to automatically determine the best configurations subject to various constraints.

In this paper, we address this problem by proposing \textit{FlexFL}, which is an FL algorithm that includes flexible control knobs that can be adjusted based on computation and communication costs. In essence, FlexFL includes three components: \textit{1)} partial computation at clients, \textit{2)} compressed parameter transmission from each client to the server, and \textit{3)} compressed parameter transmission from the server to clients. Each of these components includes its own controllable parameter to define the rate of computation (for the first component) or communication (for the second and third components). 

There are several key characteristics in FlexFL. First, the amount of computation and the amount of communication are \textit{decoupled} and can be controlled separately, allowing a high degree of freedom in control decisions to suit the current costs of different types of resources. Second, both the computation and communication rates can vary over time and they can be different for different clients and the server, which allows a high degree of \textit{system heterogeneity} and flexible resource usage depending on time-varying costs. 
Third, FlexFL includes the special case of \textit{multiple local computations}\footnote{In FlexFL, clients may perform multiple local computations (updates) with different mini-batches on the same local model parameter, and do not immediately update the parameter. This is slightly different from local updates done in the FL literature, but conceptually both approaches share similarities.} by setting the communication rate to zero in certain rounds.
Moreover, FlexFL and its analysis also allow \textit{statistical heterogeneity}, i.e., non-i.i.d. data across clients, which is commonly observed in practical FL scenarios.

We also present a convergence analysis of our FlexFL algorithm for general non-convex objectives. The resulting convergence bound provides important insights. In particular, we reveal that the convergence error increases in the residual error and decreases in the participation (computation) rate, where the residual error captures the gap between the transmitted model parameter and the computed local parameter (at each client) or received aggregated parameter (at the server). 

Finally, we formulate our control problem as \textit{stochastic optimization} over a finite time horizon, which makes decisions on the computation and communication rates over time, to minimize the convergence error subject to time-averaged cost constraints. We propose a \textit{distributed} and \textit{online} algorithm 
to approximately solve this problem.
In addition, we conduct a thorough analysis of this control algorithm and discuss its important properties and insights, based on which we explain how to balance constraint satisfaction and optimality, and also give closed-form solutions for a class of costs. 

In summary, our main contributions are as follows.
\begin{enumerate}
    \item We present an algorithm named FlexFL, which allows flexible configurations in the amount of computation at each client and the amount of communication between clients and the server. This algorithm provides a high degree of freedom in adapting the FL procedure to heterogeneous and dynamically changing resource costs.
    \item We analyze the convergence error bound of FlexFL, which reveals important insights on how the residual error and participation rate affect the convergence. This result lays out the foundation for our control algorithm.
    \item We propose a control algorithm that is derived from stochastic optimization, to approximately minimize the convergence error while satisfying constraints on the time-averaged resource cost. Our control algorithm makes decisions in an online and distributed manner, without requiring prior knowledge of system statistics.
    \item We give an in-depth analysis of our control algorithm, revealing several insights including how to adjust the trade-off between constraint satisfaction and optimality.
    \item We present experimental results on real datasets, which confirm the advantage of our proposed approach.
\end{enumerate}

\section{Related Works}

Over the past few years, efforts have been made to make FL resource-efficient, using techniques such as computing multiple local updates between communication rounds~\cite{mcmahan2017communication,yu2019parallel,Gorbunov2021Local,Haddadpour2019Local,Lin2020Dont,stich2018local}, transmitting compressed (sparse or quantized) model updates~\cite{basu2019qsparse,haddadpour2021federated,alistarh2017qsgd,bernstein2018signsgd,shlezinger2020uveqfed,reisizadeh2020fedpaq,Gradient-Sparsification,Sattler2019,albasyoni2020optimal,gorbunov2021marina,Stich2018Sparsified,The-Convergence-of-Sparsified-Gradient-Methods,stich2020error,karimireddy2019error,tang2019doublesqueeze}, and allowing only a small subset of clients to participate in each FL round~\cite{fraboni2021clustered,yang2021achieving,cho2022towards,Li2020On}. A large body of these works focuses on analyzing the convergence behavior of these algorithms, but the study of multiple local computations with partial client participation has been largely separate from compression. To our knowledge, there does not exist work that incorporates both partial client participation and the special case of no transmission (similar to multiple local computations) with general (possibly biased) compressors. 

In addition, the above works consider fixed FL configuration parameters related to communication and computation, which can be difficult to tune. To address this problem, some recent works have considered the automatic determination of communication interval~\cite{hsieh2017gaia,WangJSAC2019,MLSYS2019Jianyu}, rate of compression~\cite{li2020ggs,han2020adaptive,li2021talk,abdelmoniem2021dc2,xu2021grace,cui2022optimal}, client selection~\cite{nishio2018client,CMFL,shi2020joint,perazzone2022communication,wu2022node,wang2020optimizing,luo2021cost,luo2022tackling}, and other aspects~\cite{zhou2020cefl},  to accommodate the dynamic availability of resources. However, many of these works require a sophisticated process of estimating parameters related to the convergence bound, while some others are mostly heuristic without convergence guarantee. Moreover, none of these works consider the joint design of partial client participation and compression at both the server and clients.

\section{Federated Learning and FlexFL}

\subsection{Federated Learning Objective}

We consider an FL system with $N$ clients, where each client~$n$ has a local loss function $F_n(\bx)$ for model parameter $\bx \in \mathbb{R}^d$. The function $F_n(\bx)$ is defined on each client $n$'s local dataset, which represents the error (or loss) between the predicted output given by the model (with parameter vector $\bx$) and the ground-truth output in the training dataset. The goal of FL is to minimize the global loss function $f(\bx)$, as in:
\begin{equation}
    \textstyle\min_\bx f(\bx) := \frac{1}{N}\sum_{n=1}^N F_n(\bx),
    \label{eq:FLObj}
\end{equation}
where the average can be replaced by a weighted average if desired, but we consider the weighting coefficients to be part of $F_n(\bx)$ for simplicity. A characteristic of FL is that the local loss functions $\{F_n(\bx) : \forall n\}$ are not observed directly, because the clients' raw data are not shared. Therefore, FL needs to solve~\eqref{eq:FLObj} in a distributed manner.

\subsection{FlexFL Algorithm}

We describe our FlexFL algorithm to solve \eqref{eq:FLObj}. Similar to other FL algorithms, intermediate model parameter updates are exchanged between clients and the server, while the raw data remain private at the clients locally. The full algorithm is given in Algorithm~\ref{alg:mainAlg},
where we consider a time-slotted system and the time slots align with the iterations\footnote{We use ``time slot'' and ``iteration'' interchangeably in this paper.} in FL. We explain the main procedure of this algorithm as follows.

The algorithm includes three sets of control parameters denoted by $\{q_t^n\}$, $\{\bv_t^n\}$, and $\{\bu_t\}$. These parameters are taken as inputs by Algorithm~\ref{alg:mainAlg}, and they can be computed by our control algorithm (Algorithm~\ref{alg:controlAlg}) described later in Section~\ref{sec:Control}. We further let $\bx_t$ denote the model parameter at the beginning of each iteration $t$. 
However, we do not transmit $\bx_t$ directly between clients and the server. Instead, we transmit (possibly) compressed vectors of parameter updates, as we will see next. For the purpose of description and analysis, we assume that there are $T$ iterations in total.

\begin{algorithm}[tb]
\footnotesize
\caption{FlexFL}
\label{alg:mainAlg}

\SetKwFor{EachClient}{each client $n \leftarrow 1,\ldots,N$ in parallel:}{}{end}
\SetKwFor{Server}{the server:}{}{end}
\SetKwInput{Constants}{Constants}
\SetKwInput{Decisions}{Control parameters}

\Constants{$\eta > 0$, initial (random) model parameter $\bx_0$} 
\Decisions{$\{q_t^n\}$, $\{\bv_t^n\}$, $\{\bu_t\}$ \tcp{determined by Algorithm~\ref{alg:controlAlg} in Section~\ref{sec:Control}}}
\KwOut{ %
$\{\bx_t\}$
}

$\br_0 \leftarrow \mathbf{0}$; \label{alg-line:initialization1} \, \tcp{server residual error}

$\be_0^n \leftarrow \mathbf{0}, \forall n$;\label{alg-line:initialization2} \, \tcp{client residual error}

\For{$t \leftarrow 0, \ldots, T-1$}{
    \EachClient{\label{alg-line:loopWorkers}}{
        Sample $\!\Identity_t^n \!\sim\! \mathrm{Bernoulli}(q_t^n)$;\, \label{alg-line:randomParticipation} \tcp{randomized compute}

        $\bb_t^n \leftarrow \be_t^n - \frac{\eta \Identity_t^n}{q_t^n}\cdot \bg_n(\bx_t)$;\,
        \tcp{no compute if $\Identity_t^n=0$} \label{alg-line:localAccumulation}

        $\be^n_{t+1} \leftarrow \bb_t^n - \bv_t^n$; \,\tcp{here, $\bv_t^n$ is usually a compression of $\bb_t^n$\!\!\!\!} \label{alg-line:errorFeedbackUpdate}
        
        Send $\bv_t^n$ to the server; \, \label{alg-line:sendToServer} \tcp{transmitted local update}
    }

    \Server{}{
        $\ba_t \leftarrow \br_t + \frac{1}{N}\sum_{n=1}^N \bv_t^n$;\label{alg-line:serverAggregation}

        $\br_{t+1} \leftarrow \ba_t - \bu_t$; \,\tcp{here, $\bu_t$ is usually a compression of $\ba_t$\!\!\!\!} \label{alg-line:serverResidualErrorUpdate}
        
        Send $\bu_t$ to all clients;\, \label{alg-line:sendToClients} \tcp{transmitted global update}
    }

    \EachClient{\label{alg-line:finalSyncStart}}{    
        $\bx_{t+1} \leftarrow \bx_t + \bu_t $; \label{alg-line:globalUpdate} \, \tcp{synchronized model parameter}
    }

}
\end{algorithm}

\subsubsection{Local Computation at Clients} 
In every iteration $t$, each client $n$ computes a new stochastic gradient $\bg_n(\bx_t)$ of the local loss function $F_n(\bx_t)$ with probability $q_t^n \in (0, 1]$ (Line~\ref{alg-line:randomParticipation}). We use the identity $\Identity_t^n \in \{0,1\}$ to denote the random outcome, which is equal to one if client $n$ performs a new computation in this iteration $t$, and zero otherwise. If $\Identity_t^n = 1$, this stochastic gradient $\bg_n(\bx_t)$ is applied in Line~\ref{alg-line:localAccumulation} in the form of stochastic gradient descent (SGD) with a given learning rate of $\eta>0$. We divide the learning rate by $q_t^n$ to keep the update unbiased. If $\Identity_t^n = 0$, the last term in the right-hand side (RHS) of Line~\ref{alg-line:localAccumulation} is zero, and we do not make any update in this case. 
In practice, we \textit{do not compute $\bg_n(\bx_t)$ if $\Identity_t^n = 0$}, which is \textit{equivalent} to the update equation in Line~\ref{alg-line:localAccumulation} since the value of $\bg_n(\bx_t)$ has no effect on the subsequent updates if $\Identity_t^n = 0$. In this way, the probability $q_t^n$ controls the rate of computation, where a larger $q_t^n$ indicates that more computation is done (in expectation), consuming more computation resources, and vice versa.

\subsubsection{Client-to-Server Communication} 
Each client $n$ keeps a \textit{residual error}, which is a vector that contains portions of the \textit{changes} in the model parameter $\bx$ that have not been transmitted from the client yet. The residual error of client $n$ at the beginning of iteration $t$ is denoted by $\be_t^n$. In Line~\ref{alg-line:localAccumulation}, the new SGD update is accumulated on $\be_t^n$, giving a new temporary vector denoted by $\bb_t^n$. Then, in Line~\ref{alg-line:errorFeedbackUpdate}, the (usually sparse or quantized) vector that is transmitted to the server (i.e., $\bv_t^n$) is subtracted from $\bb_t^n$, and the remaining quantity that is not transmitted is kept in $\be_{t+1}^n$, which is the residual error at the beginning of the next iteration $t+1$. The vector $\bv_t^n$ is usually a compression result of $\bb_t^n$. We will describe in Section~\ref{sec:Control} how $\bv_t^n$ is computed, with more specific examples in Section~\ref{subsec:specificCost}.

\subsubsection{Multiple Local Computations and Decoupling}
When $\bv_t^n=\mathbf{0}$, we do not transmit in this iteration, which captures the case of multiple rounds of computation before communication happens.
Noting that $\bv_t^n = \mathbf{0}$ corresponds to no transmission by client $n$ in iteration $t$, we emphasize that we can have $\bv_t^n = \mathbf{0}$ even if $\Identity_t^n = 1$, or $\bv_t^n \neq \mathbf{0}$ even if $\Identity_t^n = 0$. In this way, the computation and communication decisions can be \textit{decoupled}. When some iterations have low computation cost but high communication cost, while other iterations have high computation cost but low communication cost, we may decide to compute in those iterations with low computation cost, and transmit in other iterations with low communication cost.

\subsubsection{Server-to-Client Communication}
After receiving the parameter updates from clients, the server averages $\{\bv_t^n : \forall n\}$ and adds the result to its own residual error $\br_t$ (Line~\ref{alg-line:serverAggregation}). If a client $n$ does not transmit any update, the server considers $\bv_t^n = \mathbf{0}$ for this client $n$ and it is still included in computing the average. Then, similar to the operation at clients, a (usually sparse or quantized) vector $\bu_t$ is transmitted to all the clients and the remaining part is kept in the residual error $\br_{t+1}$ for the next iteration $t+1$ (Line~\ref{alg-line:serverResidualErrorUpdate}). We consider a broadcast channel from the server to the clients, hence the information sent to all the clients is the same.

Finally, each client updates its current model parameter $\bx_t$ after receiving the update $\bu_t$ from the server (Line~\ref{alg-line:globalUpdate}).

\subsection{Convergence Analysis}

We analyze the convergence upper bound of Algorithm~\ref{alg:mainAlg}. First, we introduce a minimal set of assumptions that are commonly used in the literature~\cite{yang2021achieving}.

\begin{assumption}
    \label{assumption:convergence}
    We assume that the following hold, $\forall n, \bx, \by$.
    \begin{itemize}
    \item Lipschitz gradient:
    \begin{align}
        \norm{\nabla F_n(\bx) - \nabla F_n(\by)} \leq L \norm{\bx - \by}.
    \end{align}
    \item Unbiased stochastic gradient with bounded variance: 
    \begin{align}
        \!\!\!\!\!\!\!\Expectbracket{\bg_n(\bx)} \!=\! \nabla\! F_n(\bx) %
        \textrm{ and }
        \Expectbracketsizable{\big}{\normsqsizable{}{\bg_n(\bx)\! -\! \nabla\! F_n(\bx)}}\! \leq\! \sigma^2. \!\!
    \end{align}
    \item Bounded gradient divergence:
    \begin{align}
        \normsqsizable{}{\nabla F_n(\bx) - \nabla f(\bx)} \leq \epsilon^2.
    \end{align}
    \end{itemize}
\end{assumption}

The gradient divergence bound $\epsilon^2$ captures the degree of heterogeneous (i.e., non-i.i.d.) data across clients. We now introduce our main convergence result (proof is in the appendix). 

\begin{theorem}
\label{theorem:mainConvergence}
When Assumption~\ref{assumption:convergence} holds, if $\frac{1}{N}\sum_{n=1}^N \frac{1}{q_t^n} \leq p$, for all $t$, and $\eta \leq \frac{1}{4Lp}$, then Algorithm~\ref{alg:mainAlg} ensures that
\begin{align}
    &\textstyle\frac{1}{T}\sum_{t=0}^{T-1}\Expectbracketsizable{\big}{\normsqsizable{}{\nabla f(\bx_t)}}  \leq \frac{4\left(f(\bx_0) - f^*\right)}{\eta T} \nonumber \\
    &\textstyle\quad + \frac{4 L^2}{T}\sum_{t=0}^{T-1}\Expectbracketsizable{\big}{\normsqsizable{}{\br_t}} + \frac{4 L^2}{NT}\sum_{t=0}^{T-1}\sum_{n=1}^N \Expectbracketsizable{\big}{\normsqsizable{}{\be_t^n}} \nonumber\\
    &\textstyle\quad + \frac{4 \eta L (\epsilon^2 + \sigma^2)}{NT}\sum_{t=0}^{T-1} \sum_{n=1}^N \Expectbracketsizable{\Big}{\frac{1}{q_t^n}} ,
    \label{eq:convergenceMain}
\end{align}
where $f^*$ is the true minimum of $f(\bx)$, i.e., $f^* := \min_\bx f(\bx)$.
\end{theorem}

In Theorem~\ref{theorem:mainConvergence}, we capture the convergence error by the time-averaged expected squared norm of the gradient. We see that the upper bound of the convergence error increases in the squared norm of residual errors $\br_t$ and $\be_t^n$ and decreases in the probability of local computation $q_t^n$. This observation aligns with the intuition that, in general, more communication and computation can improve the convergence with respect to \textit{the number of iterations}. However, doing so would also incur \textit{higher costs} of resource usage. Therefore, we need to \textit{strike a balance between convergence error and resource cost}, after a certain number of iterations $T$. In the optimization problem presented in the next section, we aim at minimizing the convergence error under pre-defined cost constraints.

Before proceeding, we note that the decision variables $\bv_t^n$ and $\bu_t$ do not explicitly appear in the result in Theorem~\ref{theorem:mainConvergence}. For ease of presentation later, we give the following alternative upper bound that is derived from  Theorem~\ref{theorem:mainConvergence}.
Because $\br_0 = \mathbf{0}$ and $\be_0^n = \mathbf{0}$, $\forall n$, according to Algorithm~\ref{alg:mainAlg}, we can further bound the terms in \eqref{eq:convergenceMain} in the following way:
\begin{align}
    &\textstyle\sum_{t=0}^{T-1}\Expectbracketsizable{\big}{\normsqsizable{}{\br_t}} \leq \sum_{t=1}^{T}\Expectbracketsizable{\big}{\normsqsizable{}{\br_t}} \nonumber \\
    &\quad\quad\quad \textstyle= \sum_{t=0}^{T-1}\Expectbracketsizable{\big}{\normsqsizable{}{\br_t + \frac{1}{N}\sum_{n=1}^N \bv_t^n - \bu_t}}, \label{eq:alternativeBoundR}
    \\
    &\textstyle\sum_{t=0}^{T-1}\Expectbracketsizable{\big}{\normsqsizable{}{\be_t^n}} \leq \sum_{t=1}^{T}\Expectbracketsizable{\big}{\normsqsizable{}{\be_t^n}} \nonumber\\
    &\quad\quad\quad =\textstyle \sum_{t=0}^{T-1} \Expectbracketsizable{\Big}{\normsqsizable{\big}{\be_t^n - \frac{\eta \Identity_t^n}{q_t^n}\cdot \bg_n(\bx_t) - \bv_t^n }}, \forall n. \label{eq:alternativeBoundE}
\end{align}

\begin{corollary}
\label{corollary:alternativeConvergence}
Under the same conditions as in Theorem~\ref{theorem:mainConvergence}, an alternative upper bound of $\frac{1}{T}\sum_{t=0}^{T-1}\Expectbracketsizable{\big}{\normsqsizable{}{\nabla f(\bx_t)}}$ holds by replacing the corresponding terms in \eqref{eq:convergenceMain} with \eqref{eq:alternativeBoundR} and \eqref{eq:alternativeBoundE}.
\end{corollary}

\section{Control Decisions}
\label{sec:Control}

\subsection{Problem Formulation}

The goal of our decision making problem is to determine the set of control parameters $\{q_t^n\}$, $\{\bv_t^n\}$, $\{\bu_t\}$ over time, to minimize the convergence error subject to resource cost constraints. Similar to existing works \cite{WangJSAC2019,shi2020joint,perazzone2022communication,cui2022optimal,luo2021cost,luo2022tackling}, we use the convergence upper bound as an approximation to the actual error, because it is generally not possible to know exactly how different configurations affect the actual error.

\subsubsection{Instantaneous Costs}
Let $\lambda_t^n(q_t^n)$ denote the \textit{computation cost} in iteration $t$ at client $n$. Also let $\varphi_t^n(\bv_t^n)$ and $\psi_t(\bu_t)$ denote the \textit{communication cost} at client $n$ and the server, respectively, both in iteration $t$. Note that the cost functions $\lambda_t^n(\cdot)$, $\varphi_t^n(\cdot)$, and $\psi_t(\cdot)$ themselves can be different for different $t$ and $n$. That means, even if $q_t^n = q_{t'}^n$ for $t\neq t'$, we may have $\lambda_t^n(q_t^n) \neq \lambda_{t'}^n(q_{t'}^n)$, for instance. When there is no ambiguity, we omit the arguments $q_t^n$, $\bv_t^n$, and $\bu_t$ for simplicity, and only write $\lambda_t^n$, $\varphi_t^n$, and $\psi_t$ which are implicitly dependent on $q_t^n$, $\bv_t^n$, and $\bu_t$, respectively.

\subsubsection{Target Average Costs (Constraints)}
We further denote the \textit{target time-averaged computation cost} by $\tilde{\lambda}_n$ (at client $n$), and the \textit{target time-averaged communication costs} by $\tilde{\varphi}_n$ (at client $n$) and $\tilde{\psi}$ (at the server). These target costs are given as inputs to our control problem. They represent how much cost each of the clients and the server would like to spend on average for the FL task. The notion of \textit{cost} in this paper represents a generic metric. For example, it can stand for the percentage of consumed resources among all the available resources, monetary cost, energy usage, or a combination of these and other possible measures.

\subsubsection{Overall Control Problem}
With these definitions, we are ready to introduce our problem of minimizing the convergence upper bound given by Corollary~\ref{corollary:alternativeConvergence} under cost constraints.
Ignoring constants and common coefficients, we first define the following objective:
\begin{align}
\textstyle \mathcal{G} & := \textstyle \frac{L}{T}\sum_{t=0}^{T-1}\Expectbracketsizable{\Big}{\normsqsizable{\big}{\br_t + \frac{1}{N}\sum_{n=1}^N \bv_t^n - \bu_t}}\nonumber\\
&\textstyle\quad\quad + \frac{ L}{NT}\sum_{t=0}^{T-1}\sum_{n=1}^N \Expectbracketsizable{\Big}{\normsqsizable{\big}{\be_t^n - \frac{\eta \Identity_t^n}{q_t^n}\cdot \bg_n(\bx_t) - \bv_t^n }} \nonumber\\
&\textstyle\quad\quad + \frac{\eta (\epsilon^2 + \sigma^2)}{NT}\sum_{t=0}^{T-1} \sum_{n=1}^N \Expectbracketsizable{\big}{\frac{1}{q_t^n}} .
\label{eq:controlObjective}
\end{align}
Then, our overall optimization problem is as follows:
\begin{align}
    \!\!\!\!\mathrm{\textbf{P1:}} \,\, \textstyle\min_{\{q_t^n\}, \{\bv_t^n\}, \{\bu_t\}} &\quad \mathcal{G} \\
    \mathrm{s.t.}\quad\quad\quad\quad\quad\,\,\,\, &\quad \textstyle \frac{1}{T}\sum_{t=0}^{T-1} \Expectbracket{\lambda_t^n} \leq \tilde{\lambda}_n,\, \forall n \label{eq:controlProblemConstraint1}\\
    &\quad \textstyle \frac{1}{T}\sum_{t=0}^{T-1} \Expectbracket{\varphi_t^n} \leq \tilde{\varphi}_n,\, \forall n \label{eq:controlProblemConstraint2}\\
    &\quad \textstyle \frac{1}{T}\sum_{t=0}^{T-1} \Expectbracket{\psi_t} \leq \tilde{\psi}. \label{eq:controlProblemConstraint3}
\end{align}

\subsubsection{Challenges}
There are several challenges in solving the problem P1 directly. \textit{First,} there are three terms in the objective $\mathcal{G}$ defined in~\eqref{eq:controlObjective}, which have different coefficients. It is generally \textit{difficult to estimate these coefficients} as they are related to characteristics of loss functions and their stochastic gradients. \textit{Second,} in each iteration $t$ of Algorithm~\ref{alg:mainAlg}, there is a sequential order that first determines (according to $\Identity_t^n$) whether each client $n$ computes an update, then transmits the update vector from clients to the server and finally from the server to clients. Considering the second term of~\eqref{eq:controlObjective}, in practice, $\bg_n(\bx_t)$ is only computed if $\Identity_t^n = 1$, but the value of $\Identity_t^n$ is \textit{unknown} before the value of $q_t^n$ is determined. Thus, we cannot know $\bg_n(\bx_t)$ when determining $q_t^n$, which makes it impossible to use the exact value of the second term of~\eqref{eq:controlObjective} in the determination of $q_t^n$. Similarly, the value of $\frac{1}{N}\sum_{n=1}^N \bv_t^n$ in the first term of~\eqref{eq:controlObjective} is \textit{unknown} before each client $n$ has actually computed its $\bv_t^n$. \textit{Third,} the overall impact of control decisions is \textit{correlated across different iterations} through both the objective function and constraints, but we do not have prior knowledge of resource costs in practice. Therefore, we need an online algorithm that does not rely on prior knowledge.

To overcome these challenges, we first approximate P1 with three sub-problems that sequentially determine $\{q_t^n\}$, $\{\bv_t^n\}$, and $\{\bu_t\}$ in Section~\ref{subsec:sequentialApprox}. Then, we present an online algorithm for each sub-problem in Section~\ref{subsec:onlineDecision}.

\subsection{Approximation by Sequential Decision Making}
\label{subsec:sequentialApprox}

We decompose P1 into three sub-problems as follows. In each sub-problem, one set of decision variables is determined by minimizing its corresponding term in \eqref{eq:controlObjective}. We substitute expressions inside the norms using the definitions of $\bb_t^n$ and $\ba_t$ in Line~\ref{alg-line:localAccumulation} and Line~\ref{alg-line:serverAggregation} of Algorithm~\ref{alg:mainAlg}, respectively.
\begin{align}
    \mathrm{\textbf{P2.1:}}\quad \textstyle\min_{\{q_t^n\}} &\quad \textstyle \frac{1}{NT}\sum_t \sum_n \Expectbracketsizable{\big}{\frac{1}{q_t^n}}  \label{eq:controlProblemSub1} \\
    \mathrm{s.t.}\quad\quad &\quad \textrm{Constraint~\eqref{eq:controlProblemConstraint1}} . \nonumber 
    \\
    \mathrm{\textbf{P2.2:}}\quad \textstyle\min_{\{\bv_t^n\}} &\quad \textstyle \frac{1}{NT}\sum_t\sum_n \Expectbracketsizable{\big}{\normsqsizable{}{\bb_t^n - \bv_t^n }} \label{eq:controlProblemSub2}  \\
    \mathrm{s.t.}\quad\quad &\quad \textrm{Constraint~\eqref{eq:controlProblemConstraint2}}. \nonumber 
    \\
    \mathrm{\textbf{P2.3:}}\quad \textstyle\min_{\{\bu_t\}} &\quad \textstyle \frac{1}{T}\sum_t\Expectbracketsizable{\big}{\normsqsizable{}{\ba_t - \bu_t}} \label{eq:controlProblemSub3}  \\
    \mathrm{s.t.}\quad\quad &\quad \textrm{Constraint~\eqref{eq:controlProblemConstraint3}}. \nonumber
\end{align}

With this decomposition, we first solve P2.1 to obtain $\{q_t^n\}$. Then, we consider $\{q_t^n\}$ as given and solve for $\{\bv_t^n\}$ in P2.2. Finally, we consider $\{\bv_t^n\}$ as given and solve for $\{\bu_t\}$ in P2.3. We can regard this sequential decision-making procedure as an approximation to the original problem P1. The exact approximation error is difficult to analyze and is left for future work. However, we will see in Section~\ref{sec:experiments} that the solution obtained by this approximation, together with the online algorithm described in Section~\ref{subsec:onlineDecision}, provides performance gain compared to baselines in experiments.

\subsection{Online Decision Making}
\label{subsec:onlineDecision}

The problems P2.1--P2.3 are still difficult to solve directly, because both the objective functions and constraints are averaged over time, and it is difficult to predict future costs in practice. Therefore, we present an online decision making approach in the following, where the quantities $q_t^n$, $\bv_t^n$, and $\bu_t$ are determined within each iteration $t$ without knowledge of statistics in future iterations.

\subsubsection{Methodology and Challenges}
\label{subsec:LyapunovChallenges}
Our approach is based on the Lyapunov drift-plus-penalty framework~\cite{neely2010stochastic}, but with some \textit{notable differences}. First, while infinite $T$ is the primary focus in~\cite{neely2010stochastic}, we allow \textit{finite}~$T$ in this paper, both in the problem formulation (see P1 and P2 above) and in our analysis later in Section~\ref{subsec:controlAlgAnalysis}. This consideration is because, in practice, we usually train the model only for a finite number of iterations.
Second, while P2.1 can depend on an underlying system state (i.e., $\omega(t)$ defined in~\cite{neely2010stochastic}) that is independent across time $t$, we emphasize that the underlying states of P2.2 and P2.3 are both \textit{time-dependent and also dependent on previous decisions} made by the control algorithm. To see this, note that the quantity $\bb_t^n$ in P2.2 depends on the stochastic gradient computed on the model parameter $\bx_t$, and the value of $\bx_t$ is related to the decisions on $q_\tau^n$, $\bv_\tau^n$, $\bu_\tau$ made in previous iterations $\tau < t$. Similarly, $\ba_t$ in P2.3 also depends on past decisions and other random outcomes. This dependency makes it substantially harder to analyze P2.2 and P2.3, where the standard results in\cite{neely2010stochastic} no longer hold.

\subsubsection{Virtual Queues}
We define virtual queues to capture the constraints \eqref{eq:controlProblemConstraint1}--\eqref{eq:controlProblemConstraint3}. The virtual queues lengths $\Lambda_t^n$, $\Phi_t^n$, and $\Psi_t$ evolve according to the following recursions:\footnote{In practice, we may set the minimum queue length to a very small positive number instead of zero, to avoid large instantaneous costs from being incurred and added to the queue (see also the objectives of problem P3).}
\begin{align}
    \Lambda_{t+1}^n &= \max\{0, \Lambda_t^n + \lambda_t^n - \tilde{\lambda}_n\}, \quad\forall n, \label{eq:virtualQueueUpdate1} \\
    \Phi_{t+1}^n &= \max\{0, \Phi_t^n + \varphi_t^n - \tilde{\varphi}_n\}, \quad\forall n , \label{eq:virtualQueueUpdate2} \\
    \Psi_{t+1} &= \max\{0, \Psi_t + \psi_t - \tilde{\psi}\}. \label{eq:virtualQueueUpdate3}
\end{align}
Intuitively, these virtual queues capture the accumulated violation of constraints \eqref{eq:controlProblemConstraint1}--\eqref{eq:controlProblemConstraint3}. Hence, we would like to jointly minimize the objectives \eqref{eq:controlProblemSub1}--\eqref{eq:controlProblemSub3} and the virtual queue lengths. In our problem formulation (P1 and P2), the cost definitions and their constraints are separate across clients and the server, thus we can \textit{distributedly} optimize for each entity separately. An extension to settings with coupled cost constraints is possible by sharing queue length information across clients and the server.

\subsubsection{Decision Problem for Each Iteration}
Define a constant $V>0$ that will be discussed further in Section~\ref{subsec:controlAlgAnalysis}.
We have the following drift-plus-penalty minimization problems for each client $n$ (P3.1 and P3.2) and server (P3.3) in iteration~$t$.
\begin{align}
    \textstyle\mathrm{\textbf{P3.1:}}\quad\min_{q_t^n} &\textstyle\quad \textstyle \frac{V}{q_t^n} + \Lambda_t^n \big(\lambda_t^n - \tilde{\lambda}_n\big). \label{eq:controlProblemLyapunov1} 
    \\
    \textstyle\mathrm{\textbf{P3.2:}}\quad \min_{\bv_t^n} &\textstyle\quad \textstyle V\normsqsizable{}{\bb_t^n - \bv_t^n } + \Phi_t^n \left(\varphi_t^n - \tilde{\varphi}_n\right). \label{eq:controlProblemLyapunov2}  
    \\
    \textstyle\mathrm{\textbf{P3.3:}}\quad \min_{\bu_t} &\textstyle\quad \textstyle V\normsqsizable{}{\ba_t - \bu_t} + \Psi_t \big(\psi_t - \tilde{\psi}\big). \label{eq:controlProblemLyapunov3}  
\end{align}
Note that when solving P3.1--P3.3, we consider the virtual queue lengths $\Lambda_t^n$, $\Phi_t^n$, and $\Psi_t$ as well as the vectors $\bb_t^n$ and $\ba_t$ as given variables. However, these variables are inherently random due to the random noise in stochastic gradient and probabilistic client sampling, so the objectives and constraints in P1 and P2 are expressed as expectations.

The control decisions obtained from P3.1--P3.3 and virtual queue updates \eqref{eq:virtualQueueUpdate1}--\eqref{eq:virtualQueueUpdate3} are combined with Algorithm~\ref{alg:mainAlg} to provide the values of control variables. The full procedure is shown in Algorithm~\ref{alg:controlAlg}, where we may choose a non-zero initial queue length $\initqueue$ to prevent a high degree of constraint violation in initial iterations (see \eqref{eq:controlProblemLyapunov1}--\eqref{eq:controlProblemLyapunov3} and Section~\ref{subsec:controlAlgAnalysis}).\footnote{The idea of setting initial values for (virtual) queues is called place-holder backlog in \cite{neely2010stochastic}, but its original goal is to improve the performance-delay trade-off when $T\rightarrow \infty$. In contrast, we consider finite $T$ in our case, and the ``place-holder backlog'' can guarantee arbitrarily small constraint violation.}

\begin{algorithm}[tb]
\footnotesize
\caption{Online Control}
\label{alg:controlAlg}

\SetKwFor{EachClient}{each client $n \leftarrow 1,\ldots,N$ in parallel:}{}{end}
\SetKwFor{Server}{the server:}{}{end}
\SetKwInput{Constants}{Constants}
\SetKwInput{Decisions}{Control decisions}

\Constants{$V > 0$, initial queue length $\initqueue \geq 0$} 
\KwOut{$\{q_t^n\}$, $\{\bv_t^n\}$, $\{\bu_t\}$
}

Run Lines \ref{alg-line:initialization1}--\ref{alg-line:initialization2} in Algorithm~\ref{alg:mainAlg};

$\Lambda_0^n = \initqueue$, $\forall n$; $\Phi_0^n = \initqueue$, $\forall n$; $\Psi_0 = \initqueue$;

\For{$t \leftarrow 0, \ldots, T-1$}{
    \EachClient{}{
        Get $q_t^n$ from P3.1 and update $\Lambda_t^n$ using \eqref{eq:virtualQueueUpdate1};
        
        Run Lines \ref{alg-line:randomParticipation}--\ref{alg-line:localAccumulation} of Algorithm~\ref{alg:mainAlg};
        
        Get $\bv_t^n$ from P3.2 and update $\Phi_t^n$ using \eqref{eq:virtualQueueUpdate2};
        
        Run Lines \ref{alg-line:errorFeedbackUpdate}--\ref{alg-line:sendToServer} of Algorithm~\ref{alg:mainAlg};
    }
    
    \Server{}{
        Run Line \ref{alg-line:serverAggregation} of Algorithm~\ref{alg:mainAlg};

        Get $\bu_t$ from P3.3 and update $\Psi_t$ using \eqref{eq:virtualQueueUpdate3};

        Run Lines \ref{alg-line:serverResidualErrorUpdate}--\ref{alg-line:sendToClients} of Algorithm~\ref{alg:mainAlg};
    }

    Run Lines \ref{alg-line:finalSyncStart}--\ref{alg-line:globalUpdate} of Algorithm~\ref{alg:mainAlg};

}
\end{algorithm}

\subsection{Analysis of Online Control Algorithm}
\label{subsec:controlAlgAnalysis}

We discuss the optimality and constraint satisfaction of approximately solving P2.1--P2.3 via minimizing the drift-plus-penalty objectives \eqref{eq:controlProblemLyapunov1}--\eqref{eq:controlProblemLyapunov3} in P3.1--P3.3, as in Algorithm~\ref{alg:controlAlg}. Our discussion shares similarities with \cite{neely2010stochastic}. 
However, there are some key differences and challenges as discussed in Section~\ref{subsec:LyapunovChallenges}.
With a slight abuse of notation, we reuse $q_t^n, \bv_t^n, \bu_t, \lambda_t^n, \varphi_t^n, \psi_t$ to denote the control variables and their corresponding costs obtained from the \textit{solutions} of P3.1--P3.3. We first make an assumption to facilitate the analysis.
\begin{assumption}
    \label{assumption:boundedCostsAndObj}
    We assume that \emph{i)}~$q_t^n \in [\sfrac{1}{D}, 1]$, \emph{ii)}~$\normsq{\bb_t^n}, \normsq{\ba_t} \in [0, D]$, \emph{iii)}~$\lambda_t^n, \varphi_t^n, \psi_t, \tilde{\lambda}_n, \tilde{\varphi}_n, \tilde{\psi} \in [0, \sqrt{2B}]$, \emph{iv)}~$\tilde{\lambda}_n \geq \lambda_t^n(\frac{1}{D})$ (i.e., cost computed at $q_t^n = \sfrac{1}{D}$), \emph{v)}~$\varphi_t^n(\mathbf{0}) = \psi_t(\mathbf{0}) = 0$, for some $D >0, B > 0$.
\end{assumption}
In this assumption, the bound on $q_t^n$ usually holds for some $D>0$ as long as the virtual queue length $\Lambda_t^n$ is bounded. 
The rationale behind the bounds on $\normsqsizable{}{\bb_t^n}$ and $\normsqsizable{}{\ba_t}$ is that, although the residual errors are accumulated over time, they usually will not be arbitrarily large because the parameter updates get smaller when the gradient approaches zero.
Note that this is only needed for Theorems~\ref{prop:constraintViolation} and \ref{prop:objectiveBound} below, while our algorithm can still work empirically without Assumption~\ref{assumption:boundedCostsAndObj}.

\begin{theorem}
\label{prop:constraintViolation}
Under Assumption~\ref{assumption:boundedCostsAndObj}, solving P3.1--P3.3 for each~$t$ ensures the following bounds on constraint violation:
\begin{align}
    \textstyle \frac{1}{T}\sum_{t=0}^{T-1} \Expectbracket{\lambda_t^n} - \tilde{\lambda}_n & \textstyle \leq \sqrt{\frac{\initqueue^2}{T^2} + \frac{2VD+2B}{T}}- \frac{\initqueue}{T}, \label{eq:constraintViolation1} \\
    \textstyle \frac{1}{T}\sum_{t=0}^{T-1} \Expectbracket{\varphi_t^n} - \tilde{\varphi}_n & \textstyle \leq \sqrt{\frac{\initqueue^2}{T^2} + \frac{2VD+2B}{T}}- \frac{\initqueue}{T}, \label{eq:constraintViolation2} \\
    \textstyle \frac{1}{T}\sum_{t=0}^{T-1} \Expectbracket{\psi_t} - \tilde{\psi} & \textstyle \leq \sqrt{\frac{\initqueue^2}{T^2} + \frac{2VD+2B}{T}}- \frac{\initqueue}{T}. \label{eq:constraintViolation3} 
\end{align}
\end{theorem}
\begin{proof}
The Lyapunov drift of virtual queue length $\Lambda_t^n$ is
\begin{align*}
    \textstyle \Delta(\Lambda_t^n) & := \textstyle \frac{1}{2}\!\left[(\Lambda_{t+1}^n)^2 \!-\! (\Lambda_t^n)^2\right] \leq  \frac{1}{2}\!\big[\big(\Lambda_t^n \!+\! \lambda_t^n \!-\! \tilde{\lambda}_n\big)^2 \!-\! (\Lambda_t^n)^2\big] \\
    &= \textstyle \Lambda_t^n \big(\lambda_t^n - \tilde{\lambda}_n\big) + \frac{(\lambda_t^n - \tilde{\lambda}_n)^2}{2}
    \leq \textstyle \Lambda_t^n \big(\lambda_t^n - \tilde{\lambda}_n\big) + B.
\end{align*}
We have $\Lambda_t^n \big(\lambda_t^n - \tilde{\lambda}_n\big) \leq VD$, because otherwise setting $q_t^n = \frac{1}{D}$ will give a smaller value of the objective \eqref{eq:controlProblemLyapunov1} due to Assumption~\ref{assumption:boundedCostsAndObj}. Thus, $\Delta(\Lambda_t^n) \leq VD+B$.
We further note that
\begin{align*}
     \textstyle \frac{1}{2}(\Lambda_T^n)^2 - \frac{1}{2}(\Lambda_0^n)^2 = \sum_{t=0}^{T-1} \Delta(\Lambda_t^n) \leq VDT+BT.
\end{align*}
Hence, $\Lambda_T^n \leq \sqrt{(\Lambda_0^n)^2 + 2VDT+2BT}$. 
From \eqref{eq:virtualQueueUpdate1}, we have $\Lambda_t^n + \lambda_t^n - \tilde{\lambda}_n \leq \Lambda_{t+1}^n$, thus $\lambda_t^n - \tilde{\lambda}_n \leq \Lambda_{t+1}^n - \Lambda_t^n$.
This gives
\begin{align*}
     \textstyle \frac{1}{T}\sum_{t=0}^{T-1} \lambda_t^n - \tilde{\lambda}_n \leq \frac{\Lambda_T^n - \Lambda_0^n}{T} \leq \sqrt{\frac{\initqueue^2}{T^2} + \frac{2VD+2B}{T}} - \frac{\initqueue}{T},
\end{align*}
where we recall that $\Lambda_0^n = \initqueue$.
After taking expectation on both sides, we have proven \eqref{eq:constraintViolation1}. The results in \eqref{eq:constraintViolation2} and \eqref{eq:constraintViolation3} can be proven using a similar procedure.
\end{proof}

\begin{theorem}
    \label{prop:objectiveBound}
    Under Assumption~\ref{assumption:boundedCostsAndObj}, solving P3.1--P3.3 for each~$t$ gives the following bounds related to the objectives \eqref{eq:controlProblemSub1}--\eqref{eq:controlProblemSub3} of P2.1--P2.3:
    \begin{align}
        \textstyle\frac{1}{T}\sum_{t=0}^{T-1}\Expectbracket{\frac{1}{q_t^n}} &\textstyle\leq \mathrm{OPT}_{q_n} + \frac{B}{V} + \frac{\initqueue^2}{2VT}, \label{eq:objectiveLyapunovBound1} \\
        \textstyle\normsq{\bb_t^n - \bv_t^n } & \textstyle \leq \min\left\{D,\frac{\Phi_t^n \sqrt{2B}}{V}\right\},\label{eq:objectiveLyapunovBound2} \\
        \textstyle\normsq{\ba_t - \bu_t } & \textstyle \leq \min\left\{D, \frac{\Psi_t^n \sqrt{2B}}{V}\right\}.\label{eq:objectiveLyapunovBound3}
    \end{align}
    where $\mathrm{OPT}_{q_n}$ denotes the optimal value of the time-averaged objective given by a possibly randomized offline algorithm that has complete statistics of all $T$ iterations.
\end{theorem}
\begin{proof}
    The inequality \eqref{eq:objectiveLyapunovBound1} can be directly obtained from the proof of Theorem 4.8 in \cite{neely2010stochastic}. 
    To prove \eqref{eq:objectiveLyapunovBound2}, we consider an alternative choice of $\bv_t^n$, denoted by $\bv'^n_t=\bb_t^n$, which makes the objective \eqref{eq:controlProblemLyapunov2} equal to $\Phi_t^n \left(\varphi_t^n(\bb_t^n) - \tilde{\varphi}_n\right)$. Since $\bv_t^n$ is the optimal solution to P3.2, we have
    \begin{align*}
        V\normsq{\bb_t^n - \bv_t^n } + \Phi_t^n \left(\varphi_t^n(\bv_t^n) - \tilde{\varphi}_n\right) &\leq \Phi_t^n \left(\varphi_t^n(\bb_t^n) - \tilde{\varphi}_n\right).
    \end{align*}
    Rearranging gives
    \begin{align*}
        V\normsq{\bb_t^n - \bv_t^n }  &\leq \Phi_t^n \left(\varphi_t^n(\bb_t^n) - \varphi_t^n(\bv_t^n)\right).
    \end{align*}
    Then, we note that $\varphi_t^n(\bb_t^n)\in[0, \sqrt{2B}]$ and $\varphi_t^n(\bv_t^n)\in[0, \sqrt{2B}]$ according to Assumption~\ref{assumption:boundedCostsAndObj} and divide by $V$ on both sides.
    We also have $\normsq{\bb_t^n - \bv_t^n } \leq \normsq{\bb_t^n} \leq D$, because otherwise choosing $\bv_t^n=\mathbf{0}$ gives a smaller value of \eqref{eq:controlProblemLyapunov2}, where we note that $\varphi_t^n(\mathbf{0})=0$ and $\normsq{\bb_t^n} \leq D$ according to Assumption~\ref{assumption:boundedCostsAndObj} and $\Phi_t^n\geq 0$.
    Combining the above gives \eqref{eq:objectiveLyapunovBound2}. The result in \eqref{eq:objectiveLyapunovBound3} can be shown similarly.
\end{proof}

\subsubsection*{Insights}
We first discuss some important insights provided by Theorem~\ref{prop:constraintViolation}. \textit{1)} Theorem~\ref{prop:constraintViolation} shows that the constraints \eqref{eq:controlProblemConstraint1}--\eqref{eq:controlProblemConstraint3} are satisfied as $T\rightarrow \infty$. This is a desirable property of the drift-plus-penalty algorithm in time-independent settings~\cite{neely2010stochastic}. Here, we have shown that although our objectives $\normsq{\bb_t^n - \bv_t^n }$ and $\normsq{\ba_t - \bu_t}$ are correlated with past decisions over time (see Section~\ref{subsec:LyapunovChallenges}), we can still guarantee zero constraint violation when running the algorithm for a sufficiently long time. \textit{2)} By taking the derivative with respect to $\initqueue$, we can further see that the RHS of \eqref{eq:constraintViolation1}--\eqref{eq:constraintViolation3} decreases in $\initqueue$. For a finite $T$, as $\initqueue$ gets large, we will have $\sqrt{\frac{\initqueue^2}{T^2} + \frac{2VD+2B}{T}}- \frac{\initqueue}{T} \approx  \frac{\initqueue}{T} -  \frac{\initqueue}{T} = 0$. Assume that we require the RHS of \eqref{eq:constraintViolation1}--\eqref{eq:constraintViolation3} to be not larger than $\nu$, for some $\nu > 0$. For any finite $T$, we can always find a value of $\initqueue$ so that this requirement is satisfied. Therefore, the introduction of $\initqueue$ in our approach \textit{extends the constraint satisfaction} from infinite $T$, which is the primary focus of~\cite{neely2010stochastic}, \textit{to finite} $T$. \textit{3)}~When $T$ gets large, the dominant term in the RHS of \eqref{eq:constraintViolation1}--\eqref{eq:constraintViolation3} becomes $\mathcal{O}\big(\sqrt{\sfrac{V}{T}}\big)$. This shows that $\initqueue$ controls the constraint satisfaction primarily for small $T$, while the effect of $V$ becomes more prominent for large~$T$.

Next, we discuss Theorem~\ref{prop:objectiveBound}. In \eqref{eq:objectiveLyapunovBound1}, we observe an additive optimality gap of $\mathcal{O}\left(\frac{1}{V} + \frac{\initqueue^2}{VT}\right)$. This result is similar to that in~\cite{neely2010stochastic}, because the objective $\sfrac{1}{q_t^n}$ here only depends on the decision (i.e., $q_t^n$) made in the current iteration $t$. However, because $\normsq{\bb_t^n - \bv_t^n }$ and $\normsq{\ba_t - \bu_t}$ depend on past decisions, the same result does not hold for them. Nevertheless, according to the queue-length dependent bounds in \eqref{eq:objectiveLyapunovBound2} and \eqref{eq:objectiveLyapunovBound3}, the main insight that the optimality error decreases in $V$ still holds. How to obtain a queue-independent bound for the $\normsq{\bb_t^n - \bv_t^n }$ and $\normsq{\ba_t - \bu_t}$ objectives is left for future work.

Combining the above, we have the following key insight on the parameters $\initqueue$ and $V$. Increasing $\initqueue$ or decreasing $V$ improves constraint satisfaction but makes the objective function value less optimal, and vice versa. In addition, $\initqueue$ primarily affects the short-term performance with finite $T$, while $V$ affects the long-term performance. These observations are useful to guide the tuning of $\initqueue$ and $V$, so that a desired trade-off between optimality and constraint satisfaction can be achieved.

\subsection{Specific Costs and Compression Methods}
\label{subsec:specificCost}

Next, we give closed-form solutions to P3.1--P3.3 for some exemplar cost functions that have specific forms with respect to their inputs, and also discuss compression methods to obtain $\bv_t^n$ and $\bu_t$ from $\bb_t^n$ and $\ba_t$, respectively.

\subsubsection{Linear Computation Cost and Solution to P3.1}
Consider a linear computation cost defined as $\lambda_t^n = \alpha_t^n q_t^n$ for some $\alpha_t^n > 0$. The rationale behind this definition is that the expected amount of computation (e.g., number of CPU or GPU cycles) is usually proportional to the probability $q_t^n$. With this definition, we can see that the objective \eqref{eq:controlProblemLyapunov1} of P3.1 is convex in $q_t^n$. By letting the derivative of \eqref{eq:controlProblemLyapunov1} equal to zero and noting that the probability $q_t^n \in [0,1]$, we obtain the optimal solution $\hat{q}_t^n := \arg\min_{q_t^n} \frac{V}{q_t^n} + \Lambda_t^n \big(\alpha_t^n q_t^n - \tilde{\lambda}_n\big)$ as:
\begin{align}
    \textstyle \hat{q}_t^n = \min\left\{1, \sqrt{\frac{V}{\Lambda_t^n \alpha_t^n}}\right\}.
\end{align}

\subsubsection{Transmitting Compressed Update Vectors}
Regardless of the exact definition of the communication cost, the vector $\bv_t^n$ is usually derived from $\bb_t^n$. In a widely applied compression method known as \textit{top-$k$ sparsification}~\cite{Gradient-Sparsification,Sattler2019,albasyoni2020optimal}, the $k$ components of $\bb_t^n$ with the largest magnitudes are included in $\bv_t^n$ and transmitted to the server, while the remaining components that are not transmitted are kept in $\be_{t+1}^n$ for possible transmission in future iterations. When $k$ is small, the vector $\bv_t^n$ is usually represented as a sparse vector by index-value pairs, so that only the $k$ selected components are transmitted. By tuning $k$ (and the corresponding $\bv_t^n$), we can adjust the communication cost. The vector $\bu_t$ can be obtained from $\ba_t$ similarly. There are other parameter compression methods such as quantization~\cite{alistarh2017qsgd,bernstein2018signsgd,shlezinger2020uveqfed,reisizadeh2020fedpaq}. We mainly focus on top-$k$ sparsification in subsequent discussion and experiments, while noting that the same insights also apply to other compression methods.

\subsubsection{Constant-Plus-Linear Communication Cost and Solutions to P3.2--P3.3}
For the communication cost, we consider a definition that includes a constant cost portion $\beta_t^n$ whenever communication occurs. This constant portion captures the additional overhead caused by packet headers and any other necessary control information. In addition, there is a linear portion of the cost that is $\gamma_t^n$ times the amount of information transmitted.
As in top-$k$ sparsification, $\bv_t^n$ and $\bu_t$ include some components of $\bb_t^n$ and $\ba_t^n$, respectively. Then, the number of non-zero components (floating-point numbers) in $\bv_t^n$ or $\bu_t$ represents the amount of communication. 
Based on this description, the cost $\varphi_t^n$ is expressed as
\begin{align}
    \textstyle \varphi_t^n =
    \begin{cases}
    0, & \textrm{if } \norm{\bv_t^n}_0 = 0 \\
    \beta_t^n + \gamma_t^n \norm{\bv_t^n}_0, & \textrm{if } \norm{\bv_t^n}_0 > 0
    \end{cases},
    \label{eq:commSpecificCost}
\end{align}
where $\beta_t^n\geq 0$, $\gamma_t^n > 0$, and $\Vert\cdot\Vert_0$ denotes the $\ell_0$ norm that counts the number of non-zero elements in the vector. 

For a given number of non-zero components in $\bv_t^n$, it is apparent that the objective \eqref{eq:controlProblemLyapunov2} of P3.2 is minimized by choosing $\bv_t^n$ to include the $\norm{\bv_t^n}_0$ largest components (in terms of magnitude) in $\bb_t^n$. This is exactly the top-$k$ sparsification method with $k=\norm{\bv_t^n}_0$. Now, the question is how to determine $k$. Let $(b_t^n)_i$ denote the $i$-th largest component in $\bb_t^n$; we note that the change in the value of the objective \eqref{eq:controlProblemLyapunov2} from $\norm{\bv_t^n}_0=j-1$ to $\norm{\bv_t^n}_0 = j$ is $\Phi_t^n \gamma_t^n - V (b_t^n)_j^2$ 
(recall the definition of $\varphi_t^n$ in \eqref{eq:commSpecificCost}), for $j\geq 2$. The quantity $\Phi_t^n \gamma_t^n - V (b_t^n)_j^2$ does not decrease in $j$ because $\{(b_t^n)_i : \forall i\}$ is sorted in descending order. Therefore, for a specific $j$ such that $\Phi_t^n \gamma_t^n - V (b_t^n)_j^2 \geq 0$, including more than $j$ non-zero components in $\bv_t^n$ for transmission cannot make the objective \eqref{eq:controlProblemLyapunov2} smaller. Due to the discontinuity in $\varphi_t^n$ when switching from $\norm{\bv_t^n}_0 = 0$ to $\norm{\bv_t^n}_0 = 1$, we also need to check the value of the objective \eqref{eq:controlProblemLyapunov2} for the case of $\varphi_t^n = 0$. Let $j^*$ denote the smallest $j$ such that $\Phi_t^n \gamma_t^n - V (b_t^n)_j^2 \geq 0$. We have the following expression for the optimal number of components to transmit:
\begin{align}
    \textstyle \!\!\! k^* \!=\!
    \begin{cases}
        0, & \!\!\!\!\!\!\!\textrm{if } V\!\normsq{\bb_t^n}\! \leq\! V\!\normsq{\bb_t^n \!-\! \bv_t^n|_{j^*} } \!+\! \Phi_t^n \left(\beta_t^n \!+\! \gamma_t^n j^*\right) \\
        j^*, &\!\!\!\! \textrm{otherwise}
    \end{cases}\!\!,\!\!
\end{align}
where $\bv_t^n|_{j^*} \in \mathbb{R}^d$ denotes the vector that includes the $j^*$ largest components in $\bb_t^n$. Finally, the optimal $\hat{\bv}_t^n$ is obtained by setting it to $\bv_t^n|_{k^*}$, which includes the $k^*$ largest components in $\bb_t^n$. The solution to P3.3 has the same form after replacing the corresponding variables.

We conclude that P3.1--P3.3 with these simple but realistic cost definitions can be solved efficiently, while noting that our control algorithm works with other cost definitions too.

\section{\texorpdfstring{Experiments 
\vspace{-0.5em}}{Experiments}
}
\label{sec:experiments}
\begin{figure*}
    \centering
    \includegraphics[width=1\linewidth]{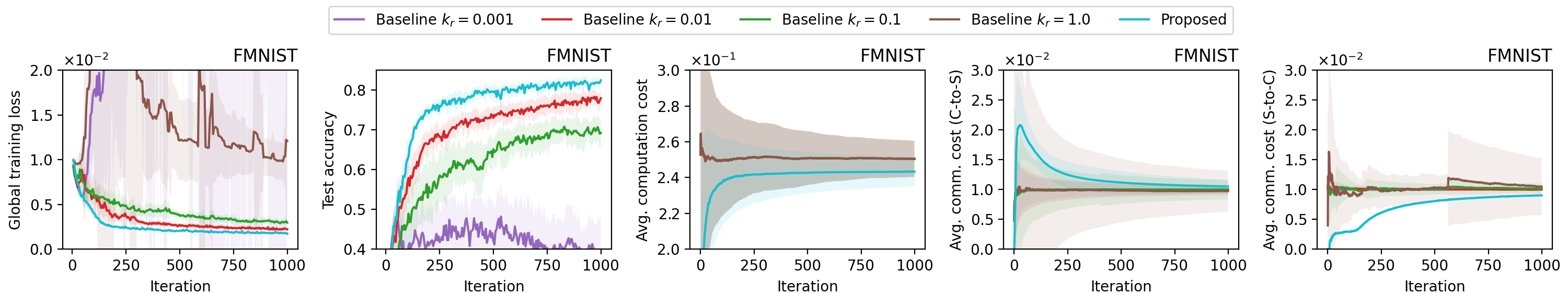}
    \includegraphics[width=1\linewidth]{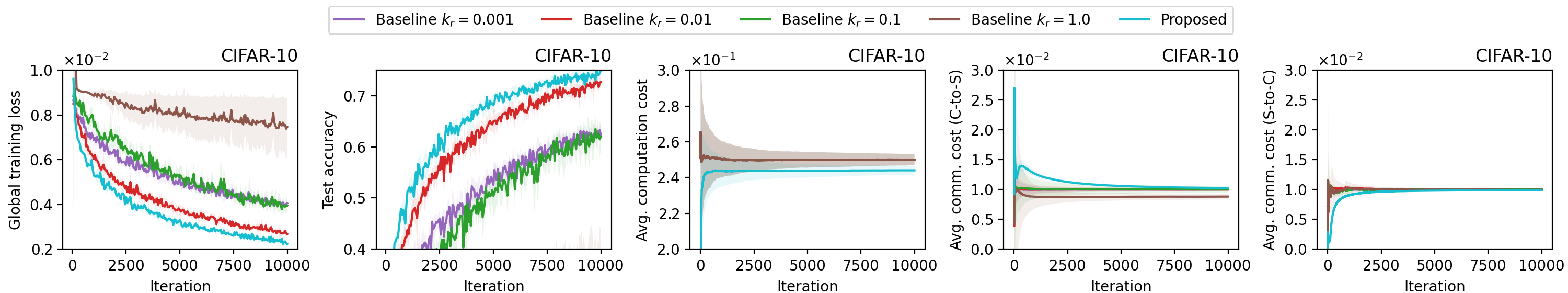}%
    \caption{FMNIST and CIFAR-10 results in comparison to baselines. The test accuracy for the baseline with $k_r := \sfrac{k}{d} = 1.0$ is below $0.4$ so it is not visible in the accuracy plot. The average cost in the plots are computed up to the iteration index on the x-axis (C = client; S = server). The costs of the baseline method with different $k_r$ largely overlap, therefore in many cases only one curve for the baseline is visible.}
    \vspace{-0.5em}
    \label{fig:ResultsBaselines}
\end{figure*}

\subsection{Setup}

\subsubsection{Datasets and Models}
We ran experiments of applying our approach to train models on image datasets. 
We consider two model and dataset combinations: \textit{1)} a two-layer neural network with a hidden layer size of $50$, trained on the Fashion-MNIST (\mbox{FMNIST}) dataset~\cite{FashionMNIST}; \textit{2)}~a convolutional neural network (CNN) with two convolutional + max-pool layers ($3\times 3$ kernel with padding, $32$ filters, followed by $2\times 2$ max-pool) and three fully-connected layers (of sizes $256$, $64$, $10$), trained on the \mbox{CIFAR-10} dataset~\cite{CIFAR10}. 
We use ReLU activation functions (except for the last layer) and Kaiming initialization~\cite{he2015delving}.
We simulate an FL system with $100$ clients. Each dataset is partitioned in a non-i.i.d. manner so that each client only has data of one class (out of all the $10$ classes), to simulate a challenging setup with high statistical heterogeneity.

\subsubsection{Costs}
The costs are defined according to the discussion in Section~\ref{subsec:specificCost}. For the computation cost, we assume that the linear coefficient $\alpha_t^n$ follows a uniform random distribution between $0$ and $1$. For the communication cost from clients to the server, we fix the constant portion to $\beta_t^n = 0.05$ to capture the overhead for headers, communication establishment, etc. The linear coefficient $\gamma_t^n$ depends on the amount of channel usage, which is related to the channel capacity. Note that the Gaussian channel capacity is $C(\Gamma_\mathrm{SNR}) := \frac{1}{2}\log_2(1+\Gamma_\mathrm{SNR})$ per channel use, where $\Gamma_\mathrm{SNR}$ denotes the signal-to-noise ratio (SNR). We define the linear portion of the communication cost $\gamma_t^n k$ as the number of channel use for transmitting $k$ components with $\Gamma_\mathrm{SNR} = \zeta$, 
normalized by the number of channel use for transmitting the entire model with $d \geq k$ components and a fixed $\Gamma_\mathrm{SNR} = 1$. This gives $\gamma_t^n := \frac{1}{k} \cdot \frac{k}{C(\zeta)}\Big/ \frac{d}{C(1)} = \frac{1}{2dC(\zeta)}$. Here, we choose 
$\zeta\sim \chi^2_2$ to simulate a Rayleigh fading channel, where we note that the square of a Rayleigh-distributed channel gain follows chi-squared distribution with a degree of freedom of $2$ (denoted by $\chi^2_2$). This definition of $\gamma_t^n$ captures the random channel condition and makes the communication cost $\varphi_t^n$ defined in \eqref{eq:commSpecificCost} to scale only with the percentage of parameter components transmitted. The communication cost from the server to clients is defined in the same way, but it is scaled down by a factor of $5$, because the downlink channel usually has higher bandwidth than the uplink channel.
In general, the randomness in these cost definitions simulate random resource costs that can be time-varying and heterogeneous across clients and the server.

\subsubsection{Baseline}
In addition to our proposed FlexFL algorithm with online control, we also consider a baseline algorithm that either transmits $k$ components or transmits nothing in each iteration $t$. When there are less than $k$ non-zero elements in $\bv_t^n$ or $\bu_t$, the baseline only transmits those non-zero elements, which can be less than $k$. To conform to the resource constraints \eqref{eq:controlProblemConstraint1}--\eqref{eq:controlProblemConstraint3}, the baseline makes a \textit{randomized} decision of whether to transmit or not in each iteration $t$, so that the expected cost in each iteration is equal to the targeted average cost (either $\tilde{\lambda}_n$, $\tilde{\varphi}_n$, or $\tilde{\psi}$). The probability $q_t^n$ is determined using a similar randomized approach by the baseline. Note, however, that when $\bv_t^n$ or $\bu_t$ has all zero entries, the expected cost of the baseline is also zero, which is smaller than constraint upper bounds ($\tilde{\varphi}_n$ or $\tilde{\psi}$). Thus, it is possible that the actual average communication cost of the baseline is slightly lower than the target (see Fig.~\ref{fig:ResultsBaselines}).
This baseline is a representative method that includes core ideas of a range of existing techniques. For example, it adapts the communication frequency based on cost constraints~\cite{hsieh2017gaia,WangJSAC2019,MLSYS2019Jianyu}, supports partial client participation (computation)~\cite{nishio2018client,CMFL,shi2020joint,perazzone2022communication,wu2022node,wang2020optimizing,luo2021cost,luo2022tackling}, and works with different sparsity values $k$~\cite{li2020ggs,han2020adaptive,li2021talk,abdelmoniem2021dc2,xu2021grace,cui2022optimal}. We use this baseline instead of specific existing methods, because we are not aware of a method that captures the same set of cost constraints as in our work, and a comparison is only meaningful if the time-averaged cost constraints are aligned.

\subsubsection{Other Parameters}
We set the time-averaged constraints to $\tilde{\lambda}_n = 0.25$, $\tilde{\varphi}_n = \tilde{\psi} = 0.01$, to simulate an environment with limited communication resources. We also set the learning rate to $\eta = 0.1$ and the default parameters of our control algorithm $V=0.02$ and $\initqueue = 1.0$.
Each setting was run with $20$ different random seeds for FMNIST and $5$ different random seeds for CIFAR-10.
In each plot, the curve shows the mean and the shaded area shows the standard deviation.

\subsection{Results}

\subsubsection{Comparing to Baseline}
We define $k_r := \sfrac{k}{d}$ as the ratio of the transmitted parameter components to the total number of components. A few observations from Fig.~\ref{fig:ResultsBaselines} are as follows. 

First, our proposed method outperforms the baseline with different $k_r$ values in both loss and accuracy values for both datasets (and models). This shows the advantage of our method that optimizes the convergence upper bound over time, which can choose different $k_r$ depending on instantaneous cost and virtual queue lengths, and it is more flexible and performs better than fixing $k_r$ as in the baseline. 

Second, the average costs of our proposed method get close to or are below their target values when the number of iterations (i.e., $T$) is large enough. This aligns with our theory in Section~\ref{subsec:controlAlgAnalysis} that has shown the constraint violation is bounded and approaches zero when $T$ gets large. It is also interesting to see that, in our proposed approach, the computation cost and server-to-client communication cost start from the lower end below the target value, while the client-to-server communication cost becomes larger than the target value in initial iterations but reduces later. This shows that the client-to-server communication is the main bottleneck with the current choice of cost and constraint parameters.

\subsubsection{Comparing Different Configurations of $V$ and $\initqueue$}
The trade-off between constraint satisfaction and optimality can be tuned by  $V$ and $\initqueue$, as discussed in Section~\ref{subsec:controlAlgAnalysis}. 
We verify this using experiments and their results are shown in Figs.~\ref{fig:ResultsDifferentV}--\ref{fig:ResultsDifferentW}. Due to space limitation, we only show the accuracy, computation cost, and client-to-server communication cost for CIFAR-10. The main observations remain the same for the other metrics and dataset. We can clearly see that the choice of $\initqueue$ mainly affects the costs in initial iterations, while the choice of $V$ has a more long-term effect. In addition, a smaller $V$ or a larger $\initqueue$ reduces the cost and gives a slightly lower accuracy. This aligns with our theoretical results in Section~\ref{subsec:controlAlgAnalysis}. 

\begin{figure}
    \centering
    \includegraphics[width=1\columnwidth]{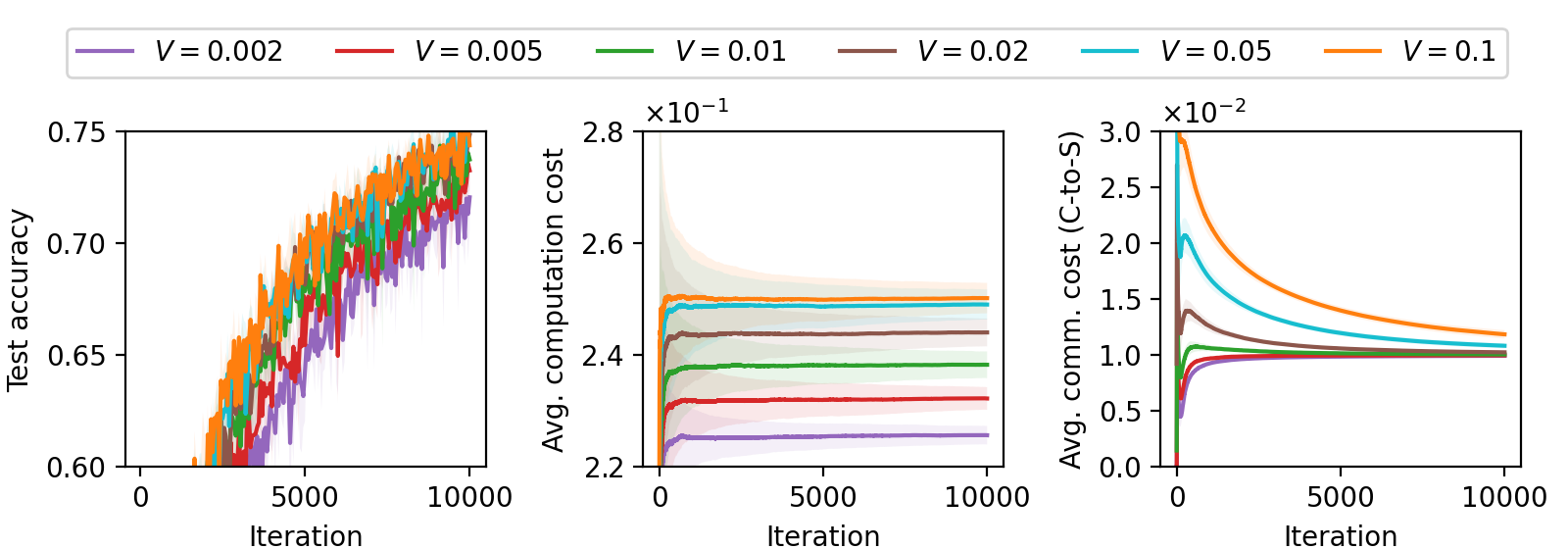}
    \caption{Proposed method with $W=1.0$ and different $V$ (CIFAR-10).}
    \label{fig:ResultsDifferentV}
\end{figure}
\begin{figure}
    \centering
    \includegraphics[width=1\columnwidth]{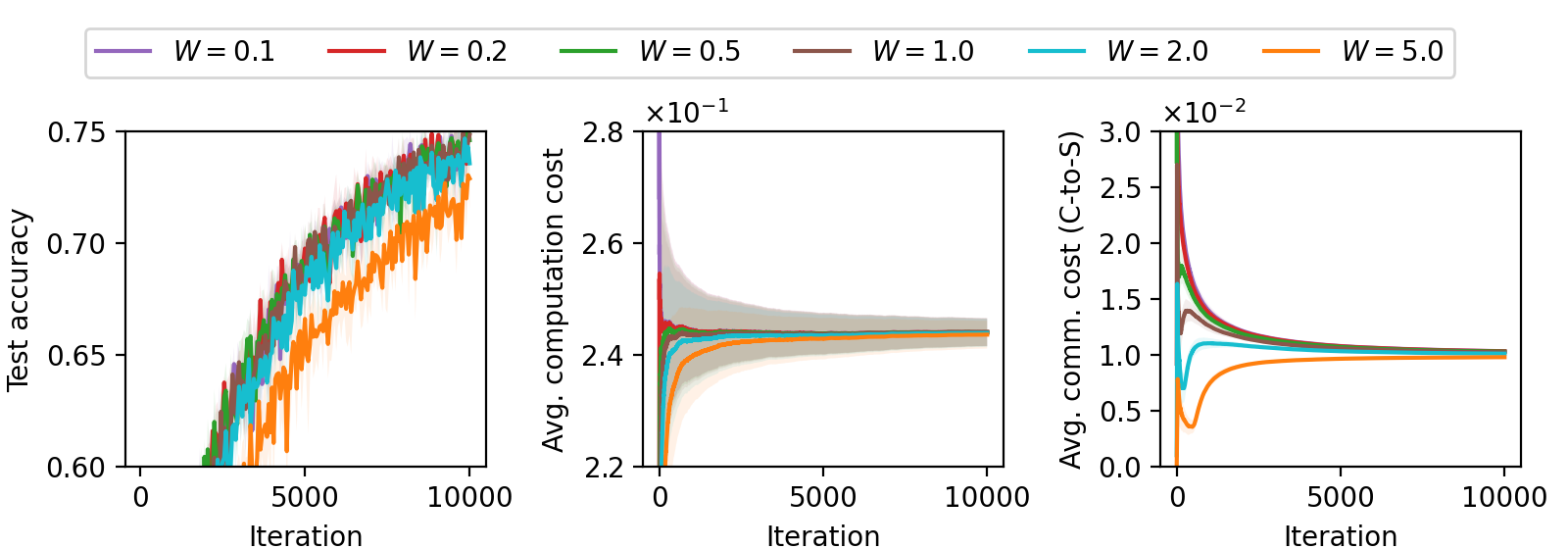}
    \caption{Proposed method with $V=0.02$ and different $W$ (CIFAR-10).}
    \label{fig:ResultsDifferentW}
\end{figure}

\section{Conclusion}
In this paper, we have proposed FlexFL and its online control algorithm. FlexFL has a set of flexible control knobs to adjust the amount of computation and communication. It includes no communication as a special case and randomly decides whether to compute in each iteration according to an adjustable probability, therefore supporting multiple local computations and partial participation. By analyzing its convergence, we have provided a theoretical foundation on how the amount of computation and communication affect the model training performance. Accordingly, we have proposed a control algorithm to automatically determine the configuration parameters of FlexFL subject to time-averaged cost constraints. The control algorithm includes useful parameters $V$ and $\initqueue$, which can be tuned to adjust the trade-off between constraint satisfaction and optimality. Our experiments show that coarsely chosen $V$ and $\initqueue$ can provide good results on two different datasets without the need of detailed tuning. If desired, $V$ and $\initqueue$ can be further tuned for fine-grained control.

There are direct extensions possible to our algorithm. For example, we may only optimize a subset of the configuration parameters in FlexFL, we can also choose different $V$ and $\initqueue$ for different types of costs and different entities. Moreover, our work provides a comprehensive methodology of optimizing multiple configuration options in FL using stochastic optimization, which can inspire future works.



\section*{Appendix: Proof of Theorem~\ref{theorem:mainConvergence}}

We first note some preliminary inequalities that will be used throughout the proof.
From Jensen's inequality, for any $\{\bz_m \in \mathbb{R}^d : m\in\{1,2,\ldots,M\}\}$, we have $\normsqsizable{\big}{ \frac{1}{M}\sum_{m=1}^M \bz_m } \leq \textstyle \frac{1}{M}\sum_{m=1}^M \normsq{\bz_m }$,
which directly gives $\normsqsizable{\big}{ \sum_{m=1}^M \bz_m } \leq \textstyle M \sum_{m=1}^M \normsq{ \bz_m }$.
Peter-Paul inequality (also known as the generalized version of Young's inequality) gives $\langle \bz_1, \bz_2 \rangle \leq \frac{\rho \normsq{ \bz_1 }}{2} + \frac{\normsq{ \bz_2 }}{2\rho}$,
for any $\rho > 0$ and any $\bz_1, \bz_2  \in \mathbb{R}^d$.
In addition, we use the notations in Algorithm~\ref{alg:mainAlg}.
We also let $\Expectsubbracket{t}{\cdot} := \Expectcond{\cdot}{\bx_t, \br_t, \{\be_t^n\}}$.
We define
\begin{align}
    \textstyle
    \tilde{\bx}_t := \bx_t + \br_t + \frac{1}{N}\sum_{n=1}^N \be_t^n.
\end{align}
From Algorithm~\ref{alg:mainAlg}, we know that
\begin{align*}
    \tilde{\bx}_{t+1} &=\textstyle \bx_{t+1} + \br_{t+1} + \frac{1}{N}\sum_{n=1}^N \be_{t+1}^n \\
    &= \textstyle(\bx_t + \bu_t) + (\ba_t - \bu_t) +\frac{1}{N}\sum_{n=1}^N(\bb_t^n - \bv_t^n) \\
    &= \textstyle\bx_t + \br_t +\frac{1}{N}\sum_{n=1}^N\Big(\be_t^n - \frac{\eta \Identity_t^n}{q_t^n}\cdot \bg_n(\bx_t)\Big) \\
    &= \textstyle\tilde{\bx}_t - \frac{\eta}{N}\sum_{n=1}^N \frac{\Identity_t^n}{q_t^n}\cdot \bg_n(\bx_t).
\end{align*}

From smoothness, we have
\begin{align}
    &\textstyle\Expectsubbracket{t}{f(\tilde{\bx}_{t+1})} \leq f(\tilde{\bx}_t) - \innerprod{\!\nabla f(\tilde{\bx}_t), \Expectsubbracket{t\!}{\frac{\eta}{N}\!\sum_{n=1}^N \frac{\Identity_t^n}{q_t^n}\cdot \bg_n(\bx_t)}} \nonumber\\
    &\textstyle\quad\quad\quad\quad\quad\quad + \frac{L}{2}\Expectsubbracketsizable{t}{\Big}{\normsqsizable{\big}{\frac{\eta}{N}\sum_{n=1}^N \frac{\Identity_t^n}{q_t^n}\cdot \bg_n(\bx_t)}}\nonumber \\
    &\textstyle\leq \! f(\tilde{\bx}_t) \!-\! \eta\!\innerprod{\nabla f(\tilde{\bx}_t), \nabla f(\bx_t)} \!+\! \frac{\eta^2\! L}{2N}\!\sum_{n=1}^N\!\Expectsubbracket{t\!\!}{\!\normsqsizable{\big}{ \frac{\Identity_t^n}{q_t^n}\!\cdot\! \bg_n(\bx_t)}} \!\!.
    \label{eq:proofMainTheorem1}
\end{align}
We consider the two terms in \eqref{eq:proofMainTheorem1} separately. We first have
\begin{align}
    &\textstyle- \eta\innerprod{\nabla f(\tilde{\bx}_t), \nabla f(\bx_t)} \nonumber\\
    &\textstyle= - \eta\innerprod{\nabla f(\tilde{\bx}_t) - \nabla f(\bx_t), \nabla f(\bx_t)} - \eta\innerprod{ \nabla f(\bx_t), \nabla f(\bx_t)} \nonumber\\
    &\textstyle\leq \frac{\eta L^2}{2}\normsq{\tilde{\bx}_t - \bx_t} + \frac{\eta}{2}\normsq{\nabla f(\bx_t)}
    - \eta\innerprod{ \nabla f(\bx_t), \nabla f(\bx_t)} \nonumber\\
    &\textstyle= \frac{\eta L^2}{2}\normsq{\br_t + \frac{1}{N}\sum_{n=1}^N \be_t^n} - \frac{\eta}{2}\normsq{\nabla f(\bx_t)} \nonumber\\
    &\textstyle\leq \eta L^2\normsq{\br_t} + \frac{\eta L^2}{N}\sum_{n=1}^N \normsq{\be_t^n} - \frac{\eta}{2}\normsq{\nabla f(\bx_t)}.
    \label{eq:proofMainTheorem2}
\end{align}
By noting that $(\Identity_t^n)^2 = \Identity_t^n$ and $\Expectsubbracket{t}{\Identity_t^n}=q_t^n$, we also have
\begin{align}
    &\textstyle\sum_{n=1}^N\Expectsubbracketsizable{t\!}{\Big}{\normsqsizable{\big}{ \frac{\Identity_t^n}{q_t^n}\cdot \bg_n(\bx_t)}}  =\sum_{n=1}^N\Expectsubbracket{t\!\!}{ \frac{\Identity_t^n}{(q_t^n)^2}}\!\cdot\!\Expectsubbracket{t\!\!}{\normsq{ \bg_n(\bx_t)}} \nonumber\\
    &\textstyle=\sum_{n=1}^N \frac{1}{q_t^n}\Expectsubbracketsizable{t}{\big}{\normsqsizable{}{ \bg_n(\bx_t)}} 
    \leq \sum_{n=1}^N \frac{1}{q_t^n}\big(\normsqsizable{}{ \nabla F_n(\bx_t)} + \sigma^2\big)  \nonumber\\
    &\textstyle \leq \sum_{n=1}^N \frac{1}{q_t^n}\big( 2\normsqsizable{}{ \nabla f(\bx_t)} + 2\epsilon^2 + \sigma^2 \big),
    \label{eq:proofMainTheorem3}
\end{align}
where the last two inequalities are due to Assumption~\ref{assumption:convergence} and also the variance relation $\Expectsubbracketsizable{t}{\big}{\normsq{\bz}} = \normsqsizable{}{\Expectbracket{\bz}} + \Expectsubbracketsizable{t}{\big}{\normsqsizable{}{\bz - \Expectsubbracketsizable{t}{}{\bz}}}$ for any random variable $\bz$.

Let $Q_t := \frac{1}{N}\sum_{n=1}^N \frac{1}{q_t^n}$. Note that we assume $Q_t \leq p$ and $\eta \leq \frac{1}{4Lp}$. Hence, $\eta \leq \frac{1}{4Lp} \leq \frac{1}{4LQ_t}$ and $- \frac{\eta}{2} + \eta^2 L Q_t \leq - \frac{\eta}{4}$. Plugging \eqref{eq:proofMainTheorem2} and \eqref{eq:proofMainTheorem3} back into \eqref{eq:proofMainTheorem1}, we obtain
\begin{align*}
    &\textstyle\Expectsubbracket{t}{f(\tilde{\bx}_{t+1})} \leq f(\tilde{\bx}_t) + \eta L^2\normsq{\br_t} + \frac{\eta L^2}{N}\sum_{n=1}^N \normsq{\be_t^n}\\
    &\textstyle\quad - \frac{\eta}{2}\normsq{\nabla f(\bx_t)} + \eta^2 L Q_t\normsq{ \nabla f(\bx_t)} + \eta^2 L Q_t (\epsilon^2 + \sigma^2)\\
    &\textstyle\leq f(\tilde{\bx}_t) + \eta L^2\normsq{\br_t} + \frac{\eta L^2}{N}\sum_{n=1}^N \normsq{\be_t^n}\\
    &\textstyle\quad - \frac{\eta}{4}\normsq{\nabla f(\bx_t)} + \eta^2 L Q_t (\epsilon^2 + \sigma^2).
\end{align*}

Taking total expectation and rearranging, we obtain
\begin{align*}
    &\textstyle\Expectbracketsizable{\big}{\normsq{\nabla f(\bx_t)}} \leq \frac{4\left(\Expectbracket{f(\tilde{\bx}_t)} - \Expectbracket{f(\tilde{\bx}_{t+1})}\right)}{\eta} \\
    &\textstyle\quad\!\! +\! 4 L^2\Expectsubbracketsizable{\!}{\big}{\normsqsizable{}{\br_t}} \!+\! \frac{4 L^2}{N}\sum_{n=1}^N \Expectsubbracketsizable{\!}{\big}{\normsqsizable{}{\be_t^n}} \!+\! 4 \eta L (\epsilon^2 \!+\! \sigma^2) \Expectsubbracket{\!\!}{Q_t}.
\end{align*}
Averaging over all $t$, we obtain the final result.
\qed

\clearpage

\bibliographystyle{IEEEtran}
\bibliography{references}

\end{document}